\documentclass[sigconf]{acmart}

\makeatletter
\renewcommand\@formatdoi[1]{\ignorespaces}
\makeatother

\usepackage{booktabs} 

\graphicspath {{figures/}}
\usepackage[latin1]{inputenc}
\usepackage{amssymb,amsmath,array}
\usepackage{graphicx}
\usepackage{verbatim}
\usepackage{multirow}
\usepackage{subfig}
\usepackage{amsmath}
\usepackage{bbold}
\usepackage{tablefootnote}

\usepackage[flushleft]{threeparttable}
\usepackage{dcolumn}
\newcolumntype{d}[1]{D{.}{.}{#1}}
\usepackage{color, xcolor}

\usepackage{verbatim}
\usepackage{multirow}
\usepackage{hhline}

\usepackage{color, xcolor}
\usepackage{tabularx}
\usepackage{colortbl}
\usepackage{dblfloatfix}

\usepackage{wrapfig}
\usepackage{lipsum}  
\usepackage{tabulary}

\setcopyright{none}

\acmDOI{}

\acmISBN{}

\acmConference[CASCON'17]{Proceedings of the 2017 conference of the Center for Advanced Studies on Collaborative Research}{November 2017}{Toronto, Ontario, Canada} 
\acmYear{2017}
\copyrightyear{2017}

\acmPrice{}

\begin{document}
\title[TrajectoryNet]{TrajectoryNet: An Embedded GPS Trajectory Representation for Point-based Classification Using Recurrent Neural Networks}

\author{Xiang Jiang}
\email{xiang.jiang@dal.ca}
\author{Erico N de Souza}
\email{erico.souza@dal.ca}
\author{Ahmad Pesaranghader}
\email{ahmad.pgh@dal.ca}
\author{Baifan Hu}
\email{baifanhu@dal.ca}
\affiliation{%
  \institution{Dalhousie University}
  \streetaddress{6299 South St}
  \city{Halifax} 
  \state{NS}
  \country{Canada}
  \postcode{B3H 4R2}
}

\author{Daniel L. Silver}
\affiliation{%
  \institution{Acadia University}
  \streetaddress{15 University Ave}
  \city{Wolfville} 
  \state{NS} 
  \country{Canada}
  \postcode{B4P 2R6}
}
\email{danny.silver@acadiau.ca}

\author{Stan Matwin}
\affiliation{%
  \institution{Dalhousie University}
  \streetaddress{6299 South St}
  \city{Halifax} 
  \state{NS}
  \country{Canada}
  \postcode{B3H 4R2}\\
  \institution{Polish Academy of Sciences}
  \city{Warsaw}
  \country{Poland}
  }
\email{stan@cs.dal.ca}

\renewcommand{\shortauthors}{X. Jiang et al.}

\begin{abstract}
Understanding and discovering knowledge from GPS (Global Positioning System) traces of human activities is an essential topic in mobility-based urban computing.
We propose TrajectoryNet---a neural network architecture for point-based trajectory classification to infer real world human transportation modes from GPS traces.
To overcome the challenge of capturing the underlying latent factors in the low-dimensional and heterogeneous feature space imposed by GPS data, we develop a novel representation that embeds the original feature space into another space that can be understood as a form of basis expansion.
We also enrich the feature space via segment-based information and use Maxout activations to improve the predictive power of Recurrent Neural Networks (RNNs).
We achieve over 98\% classification accuracy when detecting four types of transportation modes, outperforming existing models without additional sensory data or location-based prior knowledge.
\end{abstract}

%
%



\maketitle

\section{Introduction}
The advent of ubiquitous location-acquisition technologies, such as GPS and AIS (Automatic Identification System), has enabled massive collection of spatiotemporal trajectory data.
Understanding and discovering knowledge from GPS and AIS data allows us to draw a global picture of human activities and improve our relationship with the planet earth.
Among their many applications, trajectory data mining algorithms search for patterns to cluster, forecast or classify a variety of moving objects, including animals, human, cars, and vessels \cite{x585_giannotti2007trajectory,x586_zheng2015trajectory,x622_zheng2014urban,x644_batty2012smart,x645_kisilevich2010spatio,x646_nanni2006time,x793_loglisci2014mining,x792_dodge2009revealing}.
Such applications include time series forecasting tasks such as predicting the flow of crowds \cite{x138_zhang2016dnn,x555_zhang2016deep} and time series classification tasks such as detecting human transportation modes \cite{x23_zheng2008understanding} and  fishing activities \cite{x201_de2016improving,x252_FedCSIS2016546,jiang2017improving}.
These applications allow us to improve traffic management, public safety, and environmental sustainability.
In this paper, we investigate the transportation mode detection task using Recurrent Neural Networks (RNNs) that classify GPS traces into four classes (i.e., bike, car, walk and bus).

The reason why neural networks \cite{x614_byon2009real,x613_gonzalez2010automating} fail to achieve highly accurate models on this task is due to the difficulty of developing hierarchies of feature compositions in low-dimensional and heterogeneous feature space.
To address these issues, we extend our previous work \cite{jiang2017improving} on RNNs beyond ocean data to detect human transportation modes from GPS traces.
There is a significant novelty over \cite{jiang2017improving}: an in-depth analysis of the embedding method and a thorough investigation of its connections to basis expansion, piecewise function and discretization. Also, for the first time, we put forth the positive use of Maxout in GRUs as universal approximators.

We propose the TrajectoryNet method that achieves state-of-the-art performance on real world GPS transportation mode classification tasks.
The proposed TrajectoryNet differs with existing methods in that it uses embedding of GPS data, for the first time, to map the low-dimensional and heterogeneous feature space into distributed vector representations to capture the high-level semantics\footnote{Note that the term ``semantic'' in this paper refers to meaningful representations of the data, rather than geo-objects such as roads and places of interest.}.
The embedding can be viewed as a form of basis expansion that improves feature representation in the way that even a linear decision boundary in the embedding space can be mapped down to a highly nonlinear function in the original feature space.
We also employ segment-based information and Maxout activations \cite{x40_goodfellow2013maxout} to improve the predictive power of RNNs.
The TrajectoryNet achieves over 98\% and 97\% classification accuracy when detecting 4 and 7 types of transportation modes.

The rest of this paper is structured as follows:
In Section \ref{section:preliminaries} we provide definitions about trajectory data mining.
We also introduce RNNs especially Gated Recurrent Units (GRU) that will be used in this paper.
In Section \ref{section:trjnet}, we introduce the framework of the proposed model---TrajectoryNet.
We detail the segmentation method that defines the neighbourhood, embedding method and the Maxout GRU classification model.
We highlight the relationship between embedding and discretization in neural networks and provide intuitive justifications about the need of embedding for continuous features.
In Section \ref{section:exp}, we detail experiment settings and discussions on the experimental results.
We summarize the conclusion and future work in Section \ref{section:conclusion}.

\section{Preliminaries}
\label{section:preliminaries}
\subsection{Definitions}

\begin{definition}
\textup{A \textit{trajectory} \cite{x625_devogele2012mobility} is a part of movement of an object that is delimited by a given time interval $[t_{Begin}, t_{End}]$. It is a continuous function from time to space.} 
\end{definition}

 \begin{definition}
\textup{A \textit{discrete representation} \cite{x625_devogele2012mobility,x585_giannotti2007trajectory} of a trajectory is made up of a sequence of triples $S=\left \langle \left ( x_{0}, y_{0},t_{0} \right ),\dots, ( x_{k}, y_{k},t_{k} \right )  \rangle$ that represents spatio-temporal positions of the trajectory, but not providing the continuity of the movement of the object. Here $ ( x_{i}, y_{i})$ denotes the spatial coordinate at time $t_{i}$.
}
\end{definition}
The discrete representation is due to the sampling nature of location-acquisition technologies where the trajectory data are sampled at discrete timestamps.
\begin{definition}
\textup{\textit{Point-based classification} of a trajectory is the practice of learning a one-to-one mapping $S\rightarrow M$ that maps a sequence of discrete trajectory data $S=\left \langle \left ( x_{0}, y_{0},t_{0} \right ),\dots, \left ( x_{k}, y_{k},t_{k} \right )  \right \rangle$ to a corresponding sequence of labels $M=\left \langle m_{0}, \dots, m_{k} \right \rangle$ where $m_{i}$ denotes the class label of triple $( x_{i}, y_{i},t_{i})$.}
\end{definition}
 \begin{definition}
\textup{\textit{Segmentation} \cite{x585_giannotti2007trajectory} of a trajectory is to divide a trajectory into disjoint segments with some criteria such as time interval, trajectory shape or semantics that can provide richer knowledge from trajectory data.
}
\end{definition}

\begin{definition}
\textup{\textit{Discretization} of continuous features \cite{x304_dougherty1995supervised,x57_kotsiantis2006discretization,x58_garcia2013survey,x426_ali2015rough} divides the domain of the continuous attribute $D\in[l, u]$ into a set of intervals using $n$ cut-points represented by $C=\left ( c_{1}, c_{2}, \dots, c_{n} \right )$ where $c_{1} < c_{2}< \dots < c_{n}$.
The domain $D$ is divided into disjoint intervals $[l, c_{1})\cup [c_{1}, c_{2}) \cup \dots\cup [c_{n}, u]$ where $l$, $u$ are the lower and upper bounds of this attribute.
}
\end{definition}

\subsection{Recurrent Neural Networks}
RNN is a powerful model for learning from sequential data.
GPS trajectories are a type of spatiotemporal data that naturally fits into the framework of RNNs.
Unlike standard feedforward neural networks, RNNs use recurrent connections to retain the state information between different time steps.
Long short-term memory networks (LSTMs) are introduced to overcome the optimization challenges in RNNs \cite{x60_bengio1994learning,x208_hochreiter1991untersuchungen,x204_hochreiter1998vanishing} with the use of a sophisticated network structure that selectively passes information at different time steps.
There is a rich family of LSTM architectures \cite{x19_greff2015lstm,x26_chung2014empirical}, and our recent work \cite{jiang2017improving} suggests that Gated Recurrent Units (GRU) \cite{x21_cho2014learning} are better suited for point-based trajectory classification.
\textit{Gated Recurrent Units} \cite{x21_cho2014learning} defined in Equation (\ref{formula-gru}) are  a variant of LSTM.
\begin{equation}
\begin{split}
\begin{pmatrix} r_{t}\\ z_{t} \end{pmatrix} {}&= \mathrm{\sigma}
\left ( U_{g}x_{t}+  W_{g} h_{t-1} \right ) \\
\widetilde{h_{t}}{}&=\mathrm{tanh} \left (Ux_{t}+W_{c}\left ( r_{t}\odot h_{t-1} \right ) \right )\\
h_{t}{}&=(1-z_{t})\odot h_{t-1}+z_{t}\odot\widetilde{h_{t}},
\label{formula-gru}
\end{split}
\end{equation}
where $r_{t}$, $z_{t}$ are the reset and update gates to learn short and long-term memories, $\widetilde{h_{t}}$ and $h_{t}$ are candidate and final cell states at $t$, $U$ and $W$ are input-to-hidden and recurrent connections and $\odot$ denotes element-wise multiplication.
Compared with LSTMs, this results in a simplified architecture with fewer parameters that are easy to train.
This paper extends GRUs by introducing the Maxout activation function to learn more expressive memory states.


\section{Method: TrajectoryNet}
\label{section:trjnet}
Figure \ref{fig:framework} shows the framework of the proposed TrajectoryNet\footnote{The source code is available at: \url{https://github.com/xiangdal/TrajectoryNet}.}.
The GPS records are divided into segments followed by extraction of point-and-segment-based features.
It then discretizes the continuous features, embeds them into another space followed by Maxout GRUs described in Section \ref{section:maxoutGRU} for classification.

\begin{figure}[h]%
\center
    \includegraphics[width=0.48\textwidth]{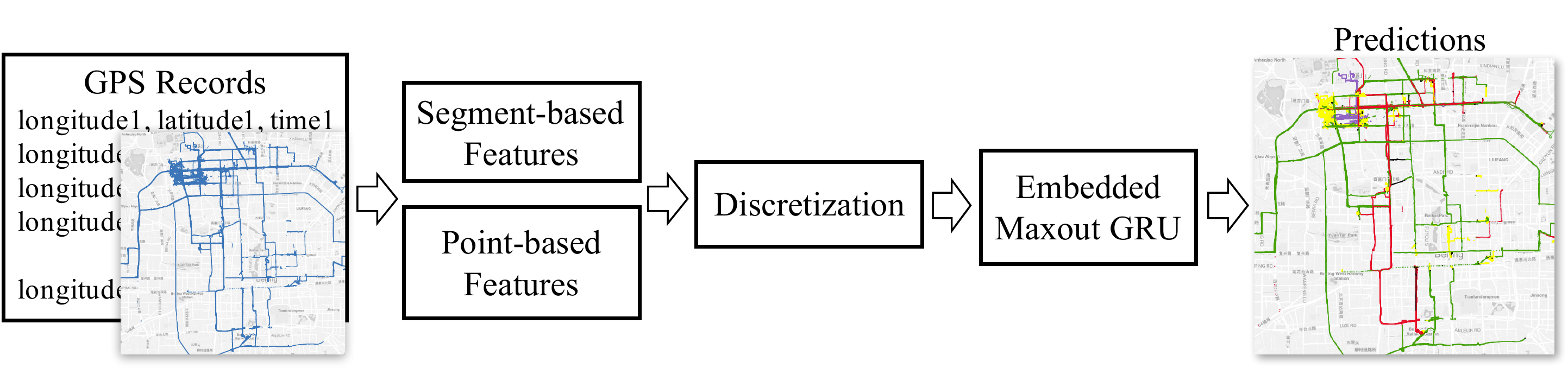}
    \caption{Framework of the proposed TrajectoryNet}%
    \label{fig:framework}
\end{figure}

\subsection{Segmentation: From Point-based Features to Segment Features}
Many transportation mode detection algorithms divide a trajectory into segments of single-mode trips and assign a single transportation mode to each segment.
Zheng et al. \cite{x23_zheng2008understanding} uses walking as the dividing criteria based on the assumption that ``people must stop and walk when changing transportation modes'' \cite{x23_zheng2008understanding}.
However, this assumption has one major practical limitation: the absence of the walking segment due to uncertainties in the sampling process could merge a trajectory that has two distinct transportation modes into a single segment.
This harms model performance by assigning a single label to this segment that contains two classes.
To address this problem, we assign transportation modes based on each discrete GPS sample instead of individual segments.
This method enriches the point-based features with segment-based features while preventing the misclassifications incurred by imprecise segmentation.

Various methods can be used for segmenting GPS trajectories, including \textit{transition}-based method \cite{x23_zheng2008understanding} that uses walking to divide trajectories, \textit{clustering}-based method \cite{x677_lee2007trajectory,x678_soares2015grasp} that measures the similarities of sub-trajectories, \textit{time}-based method \cite{x620_stenneth2011transportation} that uses equal time interval in each segment, \textit{distance}-based method \cite{x618_liao2006location} that uses equal distance traveled in each segment, \textit{bearing}-based method \cite{x270_erico_human} that measures changes in bearing orientations and \textit{window}-based method \cite{x616_miluzzo2008sensing} that has the same number of GPS samples in each segment.
This paper empirically evaluates the last four methods for their simplicity in implementation.

\subsection{Embedding: From Feature Space to Semantic Space}
\label{section:method-embedding}
\subsubsection{Motivation.}
In natural language processing, embedding is the process of converting nominal features, i.e., words, into continuous and distributed vector representations.
Compared with local representations, distributed representations have better non-local generalization \cite{x37_bengio2009learning}, are more efficient and can encode linguistic regularities and semantics.

In the context of transportation mode detection, continuous features---such as speed---are limited by their ability to capture various semantics in different applications.
The continuous features can be viewed as observations derived from underlying latent factors that carry distinct semantics in different applications.
Take speed as an example, 10km/h is ``fast'' for a running person but ``slow'' for a motorist---the same speed can take opposite ends of the spectrum depending on the context.
It is not the value of the continuous features that matters, what is more important is how we interpret the meaning.
To this end, it is desirable to develop representations that can explain the semantics of the continuous features.
In other words, we are interested in converting continuous features into a vector representation that corresponds to their meaning.

\subsubsection{The smoothness prior.}
One challenge facing the development of embeddings is that, unlike nominal features such as words, there are infinite possible values for continuous features.
To address this NP-hard problem \cite{x632_rousu2003optimal}, we use the concept of \emph{smoothness prior} that helps define the way in which we embed continuous features.
The task of transportation mode detection exhibits the property that physical attributes (e.g. speed and acceleration) behave continuously, i.e., they generally do not change abruptly in a space or time neighborhood and present some coherence.
We introduce a \emph{smoothness prior} assumption: around the value of a particular continuous attribute, e.g. speed, its semantics are more or less coherent, and changes in transportation modes do not occur all of a sudden.

\begin{theorem}
\label{theorem}
Let $F$ be a continuous conditional cumulative distribution function of a discrete (categorical) variable $Y$ given a continuous random variable $X\in[l, u]$ where $l$ and $u$ are the lower and upper bounds of $X$.
For each $\epsilon > 0$ there exists a finite partition $l\leq c_{1}<c_{2}< \dots < c_{n}\leq u$ of $[l, u]$ for $i=0,1,\dots,n-1$ such that $F(Y|c_{i+1}^{-})-F(Y|c_{i})\leq \epsilon$.
\end{theorem}
\begin{proof}
Let $\epsilon > 0$, $c_{0}=l$ and for $i\geq0$ define
$$c_{j+1}=sup \left \{ z: F(Y|z) \leq F(Y|c_{j}) + \epsilon \right \}.$$
We first prove $F(Y|c_{j+1})=F(Y|c_{j}) + \epsilon$.
We can prove $F(Y|c_{j+1})\geq F(Y|c_{j}) + \epsilon$ by contradiction, and by definition $F(Y|c_{j+1}) \leq F(Y|c_{j}) + \epsilon$, thus we conclude $F(Y|c_{j+1})=F(Y|c_{j}) + \epsilon$.
To prove $F(Y|c_{j+1})\geq F(Y|c_{j}) + \epsilon$ by contradiction, suppose $F(Y|c_{j+1})< F(Y|c_{j}) + \epsilon$, by right continuity of the conditional cumulative distribution function $F$, within the neighbourhood of $c_{j+1}$ of radius delta $\delta >0$ there exists $F(Y|c_{j+1}+\delta)< F(Y|c_{j}) + \epsilon$, which contradicts with the definition of $c_{j+1}$.
Thus, $F(Y|c_{j+1})=F(Y|c_{j}) + \epsilon$.
Next we prove $F(Y|c_{i+1}^{-})-F(Y|c_{i})\leq \epsilon$.
By definition $F(Y|c_{j+1}^{-})\leq F(Y|c_{j+1}-\delta)$ for $\delta >0$.
By definition of $c_{j+1}=sup \left \{ z: F(Y|z) \leq F(Y|c_{j}) + \epsilon \right \}$, we have $F(Y|c_{j+1}-\delta) \leq F(Y|c_{j}) + \epsilon$, which gives $F(Y|c_{j+1}^{-})\leq F(Y|c_{j+1}-\delta)\leq F(Y|c_{j}) + \epsilon$.
This completes our proof that $F(Y|c_{i+1}^{-})-F(Y|c_{i})\leq \epsilon$.
\end{proof}
Theorem \ref{theorem} justifies the smoothness prior by stating that there exists a partition, or discretization, of the feature space of random variable $X$, such that the changes in conditional cumulative distribution $F(Y|c_{i+1}^{-})-F(Y|c_{i})$ within each interval is arbitrarily small.
More concretely, in the domain of transportation mode detection, if the cumulative probability of walking $y$ given speed $v$ is $F(y|v)$, there exists a discretization such that within the speed interval $v\in[c_{i},c_{i+1})$ defined by this discretization, the cumulative probability of walking is more or less coherent, and changes in the transportation mode do not occur suddenly.
The proof of this Theorem is based on Lemma 1.1 \cite{x676_gc_theorem} in the proof of the Glivenko-Cantelli Theorem \cite{x675_vapnik2015uniform}.

The smoothness prior allows us to embed continuous features by discretizing them into intervals and embed the discretized attributes instead. 
We only discriminate among different intervals and there is no constraint that different parameterizations are required within the same interval.
Discretization has been commonly used in density estimation such as density estimation trees \cite{x676_ram2011density} that use piecewise constant function to estimate probability distributions.
It has also been used in data mining algorithms, such as C4.5 \cite{x628_quinlan2014c4} and Naive Bayes \cite{x629_yang2009discretization}.
From a Bayesian point of view, discretization allows us to use $P(Y=y|X=x^{*})$ to estimate $P(Y=y|X=x)$ where $x^{*}$ is the discretized version of input $x$ and $y$ is the label.
Yang et al. \cite{x629_yang2009discretization} shows that ``discretization is equivalent to using the true probability density function'' in the naive Bayes framework and it is empirically better to use discretization instead of unsafe parametric assumptions of the distribution.
Moreover, discretization strengthens parameter estimation through the law of large numbers where more samples are available for each interval compared with infinite values of continuous features.
In the context of point-based trajectory classification, it is difficult to ``disentangle factors of variations in the data'' \cite{x37_bengio2009learning} by developing feature compositions of the low-dimensional and dense feature space.
Thus it is reasonable to discretize and embed continuous features into vector representations for better classification.

\subsubsection{The discretization trick.}
\begin{figure}[h]%
    \centering
    \subfloat[Embedding via matrix operation]{{\includegraphics[width=6cm]{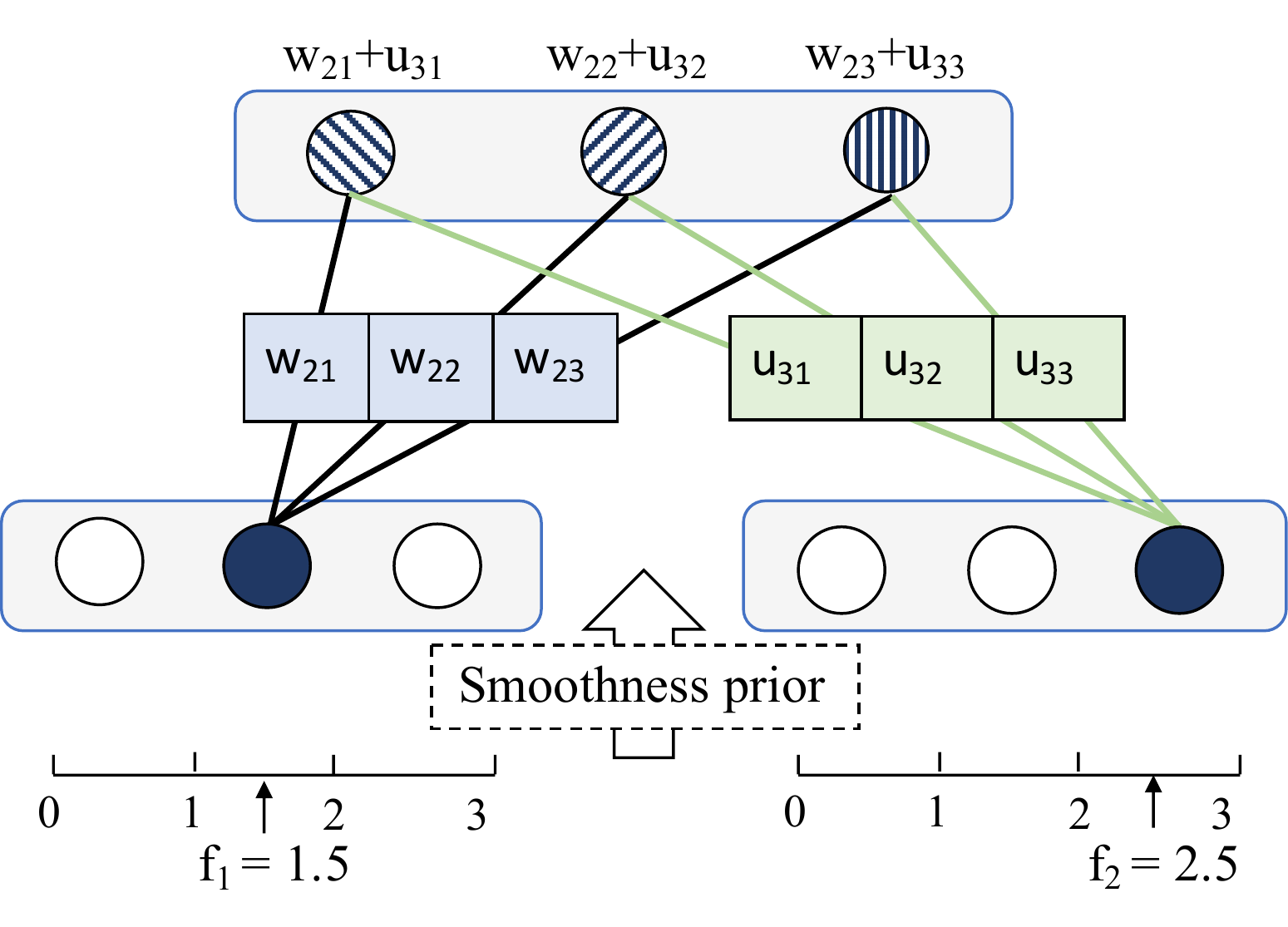} \label{fig:discretization_trick_a} }}%
    \qquad
    \hspace*{-2.5em}
    \subfloat[Embedding via lookup table]{{\includegraphics[width=7cm]{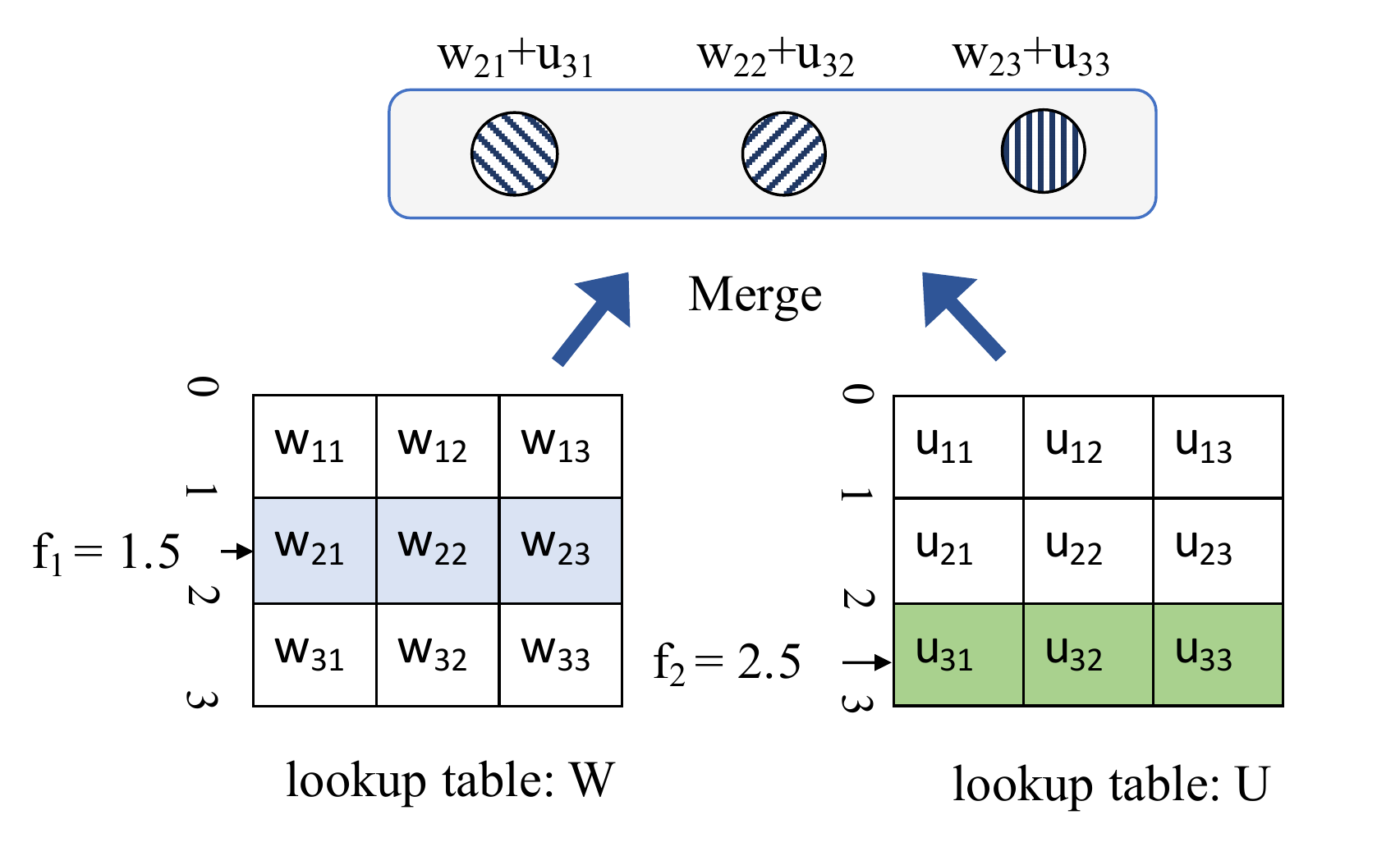} }\label{fig:discretization_trick_b}}%
    \caption{Two equivalent embedding implementations}%
\end{figure}
We introduce the ``discretization trick'' that maps discretized features into the embedding space by matrix multiplication shown in Figure \ref{fig:discretization_trick_a}.
The algorithm first converts continuous features into a one-hot vector that indicates the interval in which a feature lies in: the first feature $f_{1}=1.5$ is converted into $i=[0,1,0]$ and the second feature $f_{2}=2.5$ is converted into $j=[0,0,1]$.
The one-hot vectors are then transformed into the embedding space via matrix multiplication $[i,j]\times\begin{bmatrix}W\\U\end{bmatrix}=iW+jU=w_{e}+u_{e}$, where $w_{e}=[W_{21},W_{22},W_{23}]$ and $u_{e}=[U_{31},U_{32},U_{33}]$.
This method is equivalent with the embedding lookup approach, shown in Figure \ref{fig:discretization_trick_b}, that first defines embedding lookup tables $W$ and $U$, selects the row vectors ($w_{e}$ and $u_{e}$) based on the input values and combines them through element-wise addition.
We use the former approach in our experiments as it provides a simple means of embedding by direct matrix multiplication.
This embedding is learned through back-propagation.

\subsubsection{Discretization strategies.}
Meaningful splits are required to provide sufficient predictive power for the semantic embedding.
Three discretization strategies are studied in this paper, namely, equal-width binning, Recursive Minimal Entropy Partitioning (RMEP) and fuzzy discretization.
\emph{Equal-width binning} determines the range of each feature and then divides this range with equal-width intervals.
This method is unsupervised and straightforward to implement.
\emph{RMEP} \cite{x304_dougherty1995supervised} uses Shannon Entropy \cite{x288_shannon2001mathematical} to measure the impurity of labels within each partition and recursively partitions each feature using information gain.
RMEP improves the supervised predictive power of the resulting intervals.
The recursion terminates when a specified number of bins has been reached, or according to Minimal description length principle \cite{x289_rissanen1978modeling}.
\emph{Fuzzy discretization} aims to improve results in overlapping data by allowing a continuous value to belong to different intervals in a soft way with some (trapezoidal) membership functions \cite{x627_roy2003fuzzy}.
\subsubsection{Intuitive justifications.}
To conclude the discussions on embedding of the feature space, we highlight its beneficial characteristics.
Instead of learning layers of neural networks directly on the low-dimensional and heterogeneous feature space with shared parameters for various candidate values of each feature, embedding develops better feature representation by mapping the original features to another space that uses different embedding vectors to represent different values of each feature.
This bears some resemblance to the kernel trick. However,
adaptive piecewise basis expansions on the feature space is achieved here without explicit design of kernel functions.
Whereas the row picture of the embedding matrix represents the process of converting a continuous value into a vector, the column picture of the embedding table can be understood as a type of basis expansion.
Table \ref{fig:embedding_table} shows a sample embedding matrix that maps an attribute $D\in[0,6]$ into a 3-dimensional embedding space.
As shown in Figure \ref{fig:embedding_basis_expansion}, each embedding dimension can be viewed as a basis expansion that detects different patterns from the input space---the first dimension detects values in $[2,4]$, the second dimension is a quadratic transformation and the last dimension is an identity approximation.
Each basis expansion can be viewed as an expert that specializes in detecting different input-output relationships.
The output of the model can be interpreted as a mixture of experts \cite{x647_jacobs1997bias} that covers different regions of the decision boundary with different nonlinear functions.

\begin{figure}%
    \centering
    \subfloat[Embeddings]{{\includegraphics[width=2.7cm]{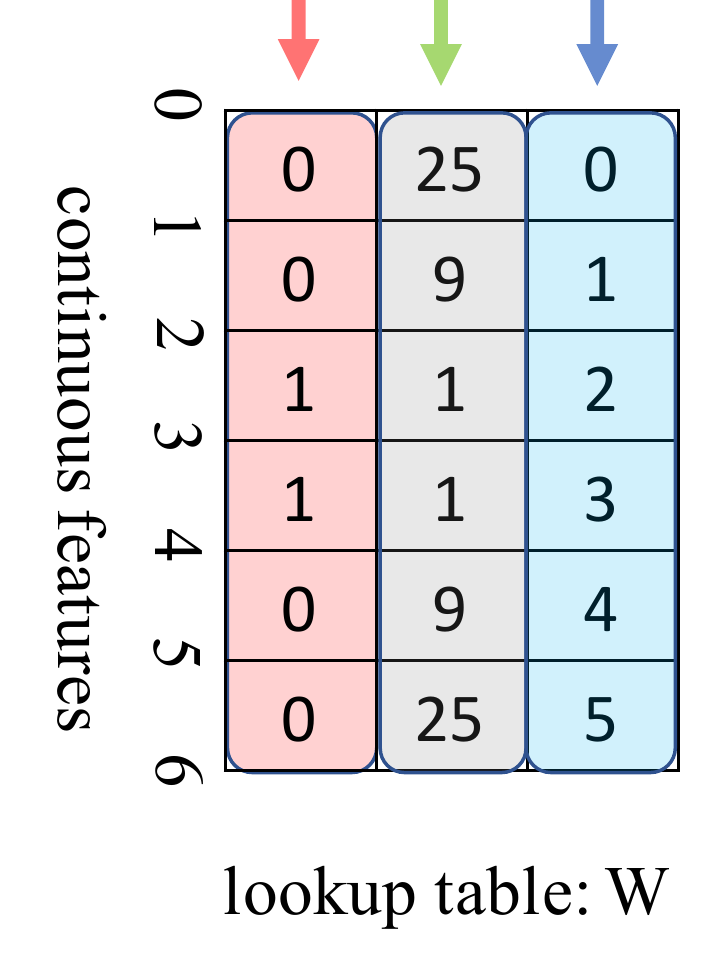} \label{fig:embedding_table} }}%
    \qquad
    \hspace*{-2.5em}
    \subfloat[Embeddings as basis expansions]{{\includegraphics[width=8cm]{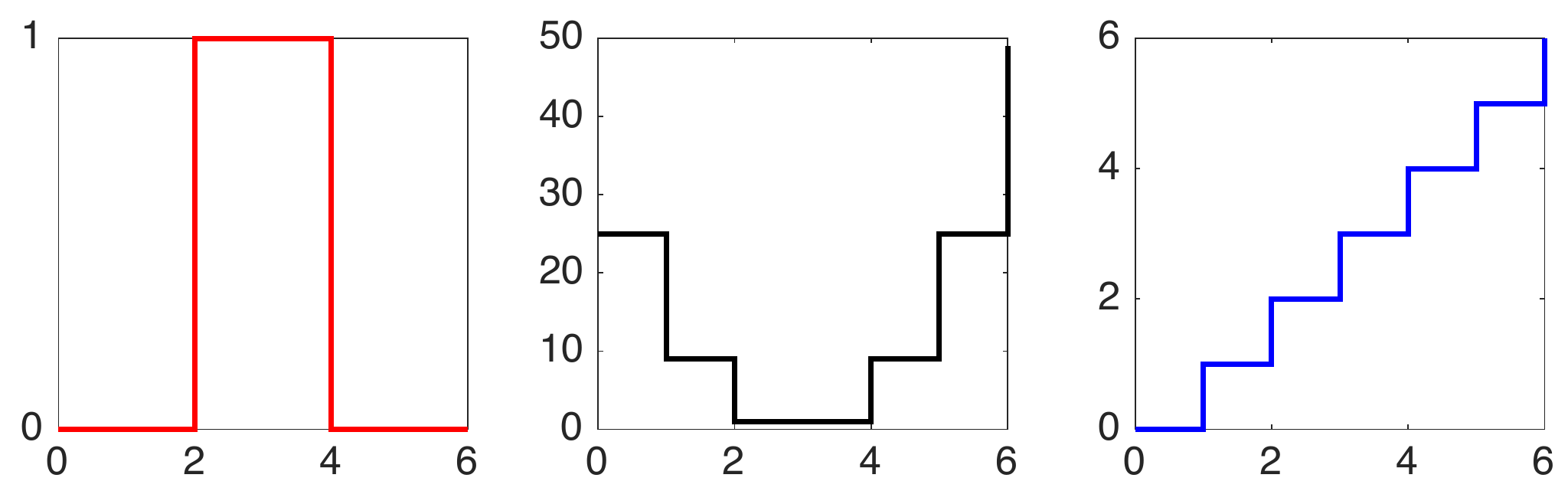} }\label{fig:embedding_basis_expansion}}%
    \caption{Basis expansion: the column picture of the embedding lookup table}%
\end{figure}

\subsection{Maxout Gated Recurrent Units}
\label{section:maxoutGRU}
We propose Maxout GRU that applies Maxout activation \cite{x40_goodfellow2013maxout} instead of hyperbolic tangent (tanh) when calculating the candidate state $\widetilde{h_{t}}$.
Maxout activation takes the maximum of a set of linear transformations resulting in an adaptive convex function.
It is more expressive than tanh because only two maxout hidden units can approximate any continuous function arbitrarily well \cite{x40_goodfellow2013maxout}.
It also has better gradient properties.
We define Maxout GRU in (\ref{formula-maxoutlstm}) that takes ``expressive'' to the next level---in addition to the learned embeddings, we also achieve more flexible memory states.
In Equation \ref{formula-maxoutlstm}, $j$ denotes the index of piecewise linear transformations, $k$ is the total number of piecewise transformations and the rest of the notations are the same with Equation \ref{formula-gru}.
\begin{equation}
\begin{split}
\begin{pmatrix} r_{t}\\ z_{t} \end{pmatrix} {}&= \mathrm{\sigma}
\left ( U_{g}x_{t}+  W_{g} h_{t-1} \right ) \\
\underset{j\in[1,k]}{\widetilde{h_{tj}}}{}{}&=U_{j}x_{t}+W_{cj}\left ( r_{t}\odot h_{t-1} \right )\\
\widetilde{h_{t}}{}&=\underset{j\in[1,k]}{\mathrm{max}} \left (\widetilde{h_{tj}} \right ) \\\
h_{t}{}&=(1-z_{t})\odot h_{t-1}+z_{t}\odot\widetilde{h_{t}}.
\label{formula-maxoutlstm}
\end{split}
\end{equation}
We use this model in conjunction with the discretization method defined in Section \ref{section:method-embedding} that maps the input $x_{t}$ into embeddings such that independent embeddings can be learned to calculate the gates and candidate memory states.
In addition to learning the semantic space, this embedding improves bias by reducing correlations between gate units and the candidate states; this is accomplished by decoupling the two lookup matrices $U_{g}$ and $U_{j}$.
\subsection{Network Architecture}

\begin{figure}
\centering
   {\includegraphics[width=6.8cm]{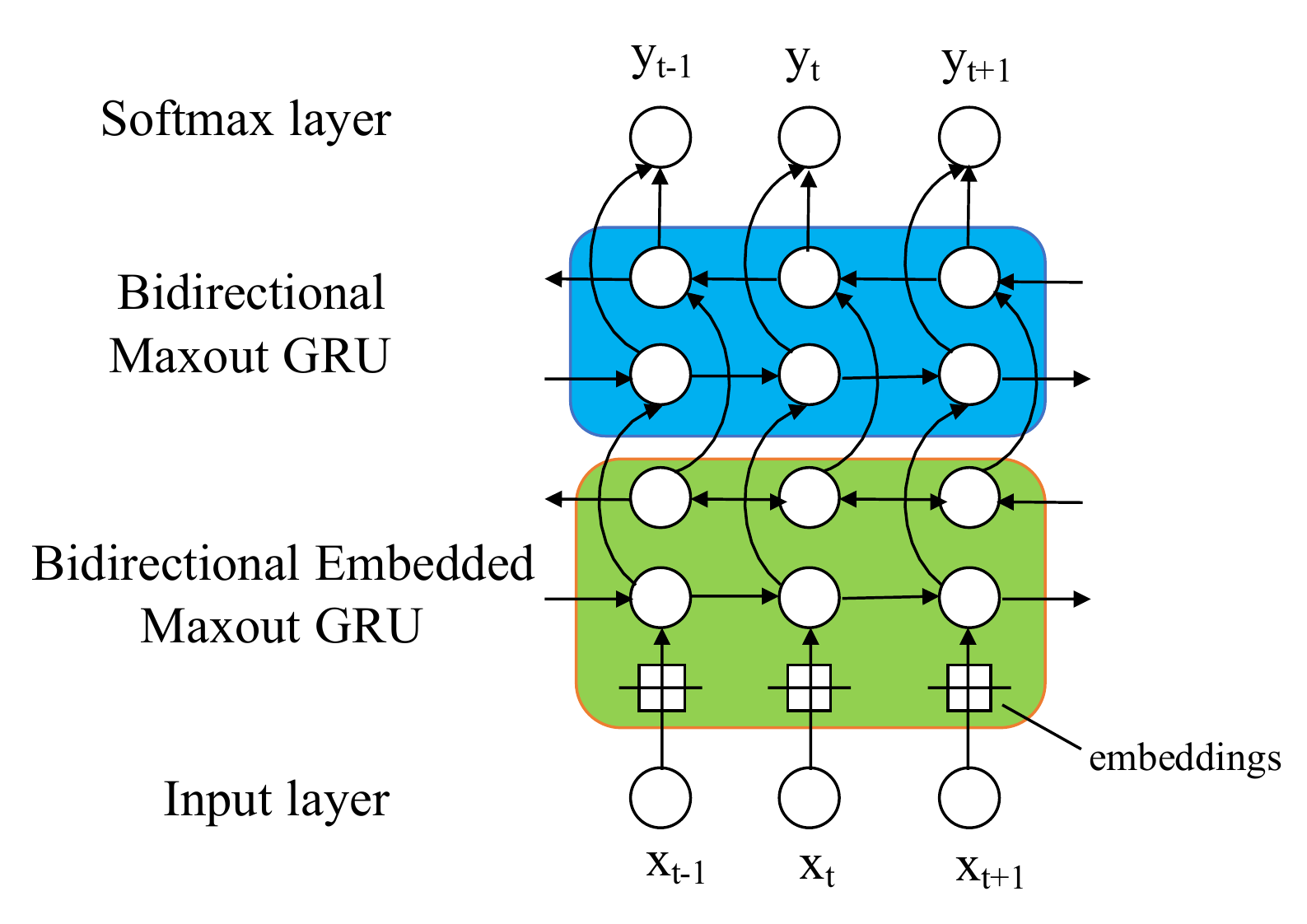} }
    \caption{TrajectoryNet architecture}%
    \label{fig:architecture}%
\end{figure}
Figure \ref{fig:architecture} shows the architecture TrajectoryNet.
The first layer learns the embedding space and the second layer learns feature compositions.
We use bidirectional GRUs instead of unidirectional GRUs to take account into bidirectional flow of information to achieve better predictive power \cite{x6_graves2013speech}.
\section{Experiments}
\label{section:exp}
\subsection{Experiment Settings}
\subsubsection{Data.}
We use the dataset collected by Zheng et al. \cite{x23_zheng2008understanding} and randomly selected 23 people's trajectories with 2,194,281 GPS records.
This dataset provides a large amount of data and a variety of transportation modes.
We focus on the four-class classification task and touch upon the seven-class results at the end.

\subsubsection{Features.}
Two main types of features are used to detect transportation modes: point-based and segment-based features.
Point-based features are associated with each individual GPS record and segment-based features are derived from the segments that aim to provide higher orders of information and regularize the noise-sensitive point-based features.
More specifically, we calculated the following location-and-user agnostic features using raw GPS records: point-based speed $v_{p}$, average speed per segment $v_{avg}$ and standard deviation of speed per segment $v_{sd}$.
However, the quality of features calculated from GPS records can be imprecise due to sensor-related reasons \cite{x625_devogele2012mobility}.
To alleviate this uncertainty, we use Hampel filter \cite{x630_hampel1974influence} to identify and filter outliers in the feature space.

\subsubsection{Network training.}
We use truncated backpropagation through time \cite{x137_werbos1990backpropagation} to optimize the cross entropy loss with Adam optimizer \cite{x8_kingma2014adam} with mini-batches.
The threshold of learning rate is 0.01 and we use validation-based early stopping to improve generalization \cite{x9_bengio2012practical}.
We use uniform initialization in the range of $[0, 0.001]$ \cite{x17_sutskever2013importance}.
We have experimented with different network structures and found that a two-layer structure with 50 hidden nodes in each layer works the best.
Each feature is divided into 20 intervals, the embedding dimension is 50 and the Maxout activation consists of 5 pieces of transformations.

\subsubsection{Evaluation.}
We use Stratified Leave One Batch Out (SLOBO) to evaluate the learned model that divides the data of 23 people into three groups: training, validation and testing set that contain trajectories of 16, 1 and 6 people, respectively.
Because the trajectories from different person may contain varying proportions of transportation modes, we select the three subgroups with the objective of getting similar proportions of transportation modes that are representative of the population.
The stratification reduces variance in the training process and prevents validation-based early stopping from poor generalization.
Because each trajectory has its own intrinsic characteristics, SLOBO also reflects the ability to generalize beyond trajectories of new individuals.

Four methods are selected to measure the model performance: point-based classification accuracy $A_{point}$, distance-based accuracy $A_{distance}$, cross-entropy loss $E_{H}$ and average F1 score $A_{F1}$.
The descriptions are shown in Table \ref{tab:evaluation_measures}.

\begin{table}[H]
\centering
\caption{Description of evaluation measures}
\label{tab:evaluation_measures}
\scalebox{0.8}{
\begin{tabular}{ll}
\toprule
 Measure & Description                                                                                                                             \\ \hline
$A_{point}$                   &  accuracy based on the number of GPS samples                                                                                              \\
  $A_{distance}$                  & \begin{tabular}[c]{@{}l@{}} accuracy based on distances traveled,  for comparison with \cite{x23_zheng2008understanding}\end{tabular} \\
$E_{H}$                   & insights into the model training process (learning curve)                                                                       \\
$A_{F1}$                   & performance measure when different classes are imbalanced                                                                          \\ \bottomrule
\end{tabular}}
\end{table}
\subsection{Results and Discussion}
\subsubsection{Comparison with baseline methods.}
Table \ref{tab-F1} shows the F1 score and accuracy on the test data with distance-based evaluation measures.
The most frequent classification errors are between car and bus as well as walk and bus.
Overall, TrajectoryNet achieves an encouraging 98\% accuracy and outperforms existing baseline methods---including Decision Tree, Support Vector Machine, Naive Bayes, and Conditional Random Field---on the same dataset by a large margin.
The improvement in F1 score over the decision tree based framework proposed by Zheng et al. \cite{x23_zheng2008understanding,x24_zheng2008learning} is 31\%, 16\%, 22\% and 22\% for each class.
This demonstrates the overall effectiveness of the proposed TrajectoryNet.
We further analyze the effect of the individual components of the proposed TrajectoryNet in the rest of this section.
\begin{table*}[]
\centering
\caption{TrajectoryNet performance on the test data}
\label{tab-F1}
\scalebox{1}{
\begin{tabular}{lccccccc}
\toprule
                      & \multicolumn{5}{c}{F1 score\tablefootnote{F1 scores denoted by ``-'' are not available in the referred papers.}}            &  & \multicolumn{1}{c}{\multirow{2}{*}{Accuracy}} \\ \cline{2-6}
                      & bike  & car   & walk  & bus   & average &  & \multicolumn{1}{c}{}                          \\ \hline
TrajectoryNet (ours) & 0.988 & 0.980 & 0.972 & 0.980 & 0.980   &  & 0.979                                         \\
Decision Tree \cite{x23_zheng2008understanding}       & 0.675 & 0.814 & 0.757 & 0.748 & 0.749   &  & 0.762                                         \\
Support Vector Machine \cite{x24_zheng2008learning}       & - & - & - & - & -   &  & 0.462                                         \\
Naive Bayes \cite{x24_zheng2008learning}       & - & - & - & - & -   &  & 0.523                                        \\
Conditional Random Field \cite{x24_zheng2008learning}       & - & - & - & - & -   &  & 0.544                                         \\ \bottomrule
\end{tabular}}
\end{table*}

\subsubsection{The effect of embedding.}
Figure \ref{fig:learning_curve_activations} shows the learning curves of GRUs with and without embedding.
All three embedded GRUs have a better (lower) cross entropy loss $E_{H}$ and converge faster than conventional GRUs: the $E_{H}$ embedded GRU achieves in only two epochs is better than the best $E_{H}$ that GRU achieves in over 70 epochs.
This speedup is because embedding decouples the dependencies between different input values that makes optimization more straightforward.
The embedded GRUs are more stable compared with GRU that suffers from exploding gradients in epoch 20 (ReLU) and 55 (tanh).
This improvement is due to the fact that even a linear decision boundary in the embedding space can be mapped down to a highly nonlinear function in the original space while being easier to optimize.
It reinforces our claim that the representation learned by embedding improves the predictive power of RNNs.

\subsubsection{The role of different activation functions.}
We also highlight the role of different activation functions in Figure \ref{fig:learning_curve_activations}.
We find that Maxout activation functions converge better no matter whether embedding is used.
This demonstrates that Maxout activation can learn more flexible memory states in GRUs.
Note that the fluctuation in the learning curve is a result of Adam stochastic gradient descent with mini-batches.
Whereas the tanh and ReLU suffer from exploding gradients without the use of embedding, Maxout activation does not suffer from exploding gradients as it has better gradient properties---piecewise linear.

\begin{figure}[h]%
\center
    \includegraphics[width=8cm]{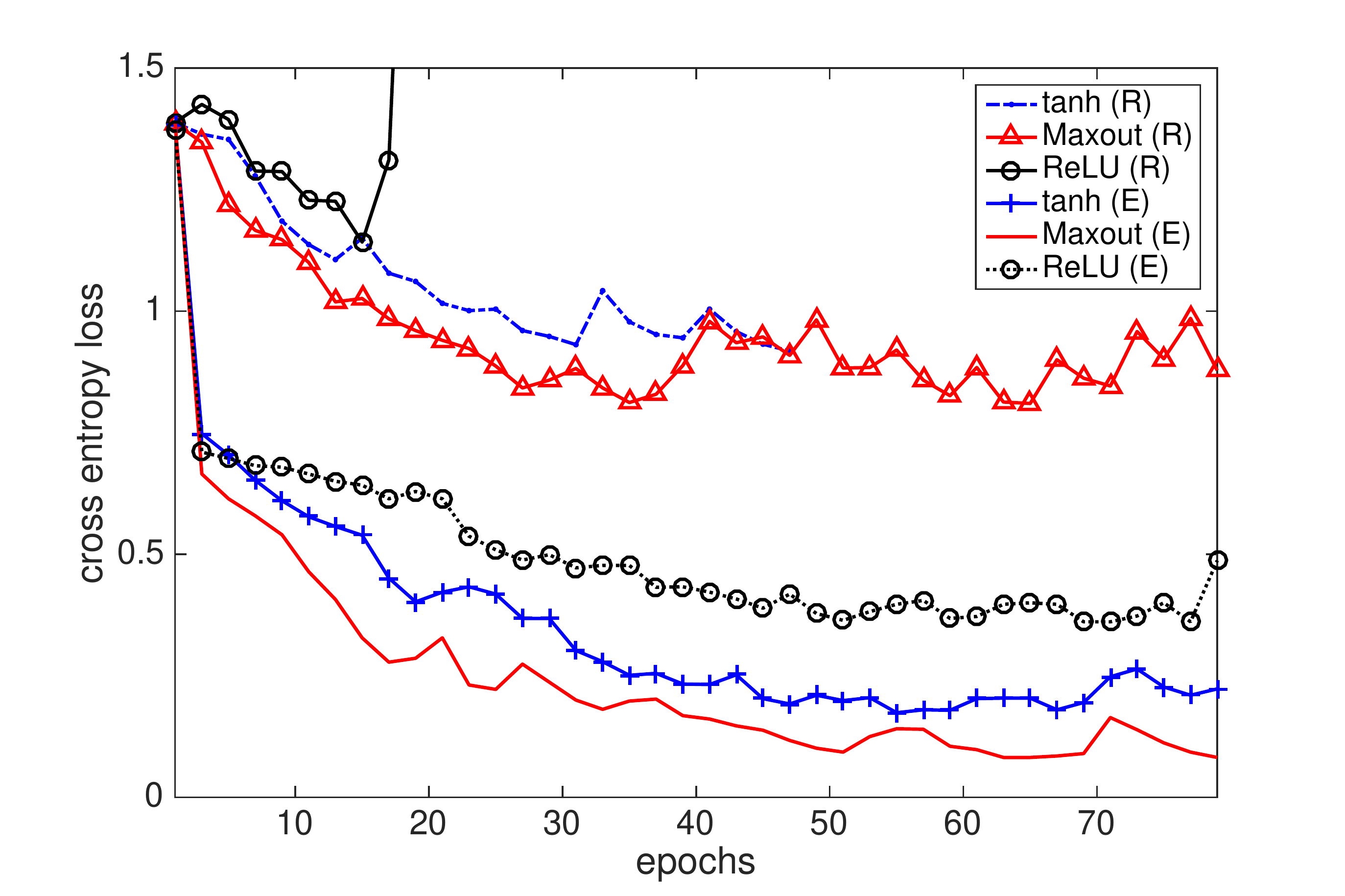} 
    \caption{Learning curves of GRU methods with different activation functions. (R) denotes GRU without embedding and (E) denotes embedded GRU.}%
    \label{fig:learning_curve_activations}
\end{figure}

\subsubsection{The need for segment-based features.}
Table \ref{tab:feature-combination} shows the classification results when using different combinations of features, and demonstrates that the segment-based features can improve the model performance.
This is consistent with our findings about the probability density functions that different features complement each other to provide better separation among various transportation modes.
This also demonstrates that the segment-based features can enrich the point-based feature.
Note that we used the same network architecture throughout these experiments.
Also note that methods marked by $\star$ in the Tables are significantly better than methods without $\star$ at a confidence level of 95\% when performing t-test on $E_{H}$, and the differences are not significant among all methods marked by $\star$ at the 95\% confidence level.

\begin{table}[h]

\centering
\caption{Forward feature selection}
\label{tab:feature-combination}
\begin{tabular}{llll}
\toprule
                  & $E_{H}$        & $A_{point}$   & $A_{F1}$      \\ \hline
$v_{p}$                   & 0.24          & 0.91          & 0.92          \\
$v_{p}, v_{avg}$          & 0.11          & 0.97          & 0.97          \\
$v_{p}, v_{avg}, v_{sd}$$\star$ & 0.08 & 0.98 & 0.98 \\ \bottomrule
\end{tabular}
\end{table}

\subsubsection{The effect of segmentation.}
The selection of segmentation strategy can be expected to affect the model performance.
Table \ref{tab:segmentation} shows results of different segmentation methods.
The distance-based method performs the worst because different transportation modes travel at different speeds that result in a varying amount of samples within each segment.
This further affects the quality of segment-based features as a result of different sampling complexities in different transportation modes.
The rest of the segmentation methods are not statistically different, and the bearing-based segmentation method has the best average performance.
Please refer to \cite{x270_erico_human} for further comparison and discussions.
\begin{table}[h]
\centering
\caption{The effect of segmentation}
\label{tab:segmentation}
\begin{tabular}{llll}
\toprule
        & $E_{H}$ & $A_{point}$ & $A_{F1}$ \\ \hline
time$\star$        &   0.11                     &   0.95                &   0.95                \\
distance     &     0.25                    &     0.86              &      0.86             \\ 
bearing$\star$ & 0.08                  & 0.98             & 0.98             \\
window$\star$     & 0.09                  & 0.97             & 0.97             \\
\bottomrule
\end{tabular}
\end{table}

\subsubsection{The effect of discretization.}
Table \ref{tab-discretization} shows the results of different discretization methods.
It may seem surprising that \emph{equal-width binning} works on a par with entropy-based method, but we find equal-width binning robust, easy to implement and easy to train.
\emph{Fuzzy coding} \cite{x627_roy2003fuzzy} does not work as well as the previous two methods.
Owing to the trapezoidal fuzzy function used in this experiment, the model is forced to learn a weighted sum of two embedding vectors at the same time which makes optimization difficult.
This means given the smoothness prior and proper granularities of the partitions, overlapping interval is not a necessity to learn good models.
As shown in Table \ref{tab:discretization-granularity}, we experimented with various discretization granularities and found that the model was improved significantly when increasing the number of intervals from 10 to 20, but this choice made little difference when the number of intervals is between 20 to 50 for each feature.

\begin{table}[h]
\centering
\caption{The effect of discretization strategies}
\label{tab-discretization}
\begin{tabular}{llll}
\toprule
       & $E_{H}$ & $A_{point}$ & $A_{F1}$ \\ \hline
width$\star$      & 0.08         & 0.98    & 0.98    \\
entropy$\star$          & 0.19         & 0.95     & 0.95    \\
fuzzy            & 0.14         & 0.96    & 0.96    \\ \bottomrule
\end{tabular}
\end{table}
\begin{table}[h]
\centering
\caption{The effect of discretization granularities}
\label{tab:discretization-granularity}
\begin{tabular}{llllll}
\toprule
\# intervals   & 10         & 20$^{\star}$         & 30$^{\star}$         & 40$^{\star}$         & 50$^{\star}$         \\ \hline
$E_{H}$        & $0.168$    & $0.076$    & $0.062$    & $0.068$    & $0.070$    \\
standard error & $0.004$ & $0.013$ & $0.005$ & $0.007$ & $0.012$ \\ \bottomrule
\end{tabular}
\end{table}

\subsubsection{Seven-Class Classification.}
To further validate the effectiveness of the TrajectoryNet, we undertake a more challenging task: classifying GPS records into seven classes, namely train, car, bus, subway, airplane, and bike.
This is a much more challenging task and we achieve 97.3\% point-based classification accuracy and 93.0\% average F1 score, as shown in Table \ref{tab:seven_class}.
Compared with the four-class classification task, the reduction in model performance in mainly due to the insufficient amount of training examples for the two classes: subway and airplane.

\begin{table}[]
\centering
\caption{Confusion matrix to detect 7 transportation modes ($A_{point}=97.3\%$)}
\label{tab:seven_class}
\scalebox{0.72}{
\begin{tabulary}{1.0\textwidth}{lrrrrrrrr}
\toprule
\multirow{2}{*}{Target} & \multicolumn{7}{c}{Prediction}                                                                                                                                                                  & \multicolumn{1}{l}{\multirow{2}{*}{Recall}} \\ \cline{2-8}
                        & \multicolumn{1}{c}{train} & \multicolumn{1}{c}{car} & \multicolumn{1}{c}{walk} & \multicolumn{1}{c}{bus} & \multicolumn{1}{c}{subway} & \multicolumn{1}{c}{airplane} & \multicolumn{1}{c}{bike} & \multicolumn{1}{l}{}                                 \\ \hline
train                   & \textbf{380563}           & 889                     & 853                      & 67                      & 514                        & 1                            & 0                        & 0.994                                                \\
car                     & 83                        & \textbf{66309}          & 261                      & 938                     & 852                        & 0                            & 300                      & 0.965                                                \\
walk                    & 82                        & 260                     & \textbf{139026}          & 2397                    & 784                        & 26                           & 1217                     & 0.967                                                \\
bus                     & 29                        & 298                     & 2920                     & \textbf{121853}         & 76                         & 12                           & 421                      & 0.970                                                \\
subway                  & 1394                      & 1405                    & 6083                     & 142                     & \textbf{24875}             & 10                           & 45                       & 0.733                                                \\
airplane                & 4                         & 0                       & 43                       & 23                      & 0                          & \textbf{1979}                & 13                       & 0.960                                                \\
bike                    & 0                         & 0                       & 95                       & 26                      & 14                         & 0                            & \textbf{64194}           & 0.998                                                \\ \hline
Precision      & 0.996                     & 0.959                   & 0.931                    & 0.971                   & 0.917                      & 0.976                        & 0.970                    &                                                      \\
F1 score       & 0.995                     & 0.962                   & 0.949                    & 0.971                   & 0.815                      & 0.837                        & 0.984                    &                                                      \\ \bottomrule
\end{tabulary}}
\end{table}

\subsubsection{Visualizing classification results.}
Figure \ref{fig:imperfect-label} shows the predictions in the test data.
The transportation modes, bike, car, walk and bus, are colored in purple, red, yellow, and green, respectively.
The misclassified GPS records are colored in black.
This figure shows the overall effectiveness of the TrajectoryNet---only a very small fraction of data are misclassified.
To further investigate the classification errors incurred from the TrajectoryNet, we highlight three scenarios in the data, labeled as A, B and C in the figure.
Scenario A misclassified car into bus and scenario B misclassified bus into car.
These are the most common classification errors from the TrajectoryNet; it is challenging to distinguish these two classes as they sometimes manifest similar features.
Scenario C are hard to notice from the visualization as they only occur during the transition between two transportation modes.
This constitutes another common type of misclassification but only happens in rare occasions implying that the proposed TrajectoryNet is sensitive at detecting transitions between different modes.

\begin{figure*}
  \centering
  \includegraphics[width=0.7\textwidth]{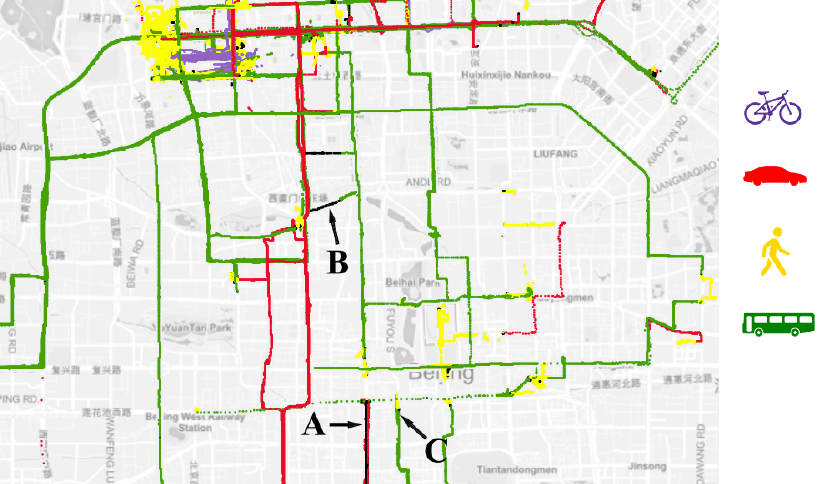}
  \captionof{figure}{Visualizing classification results}
  \label{fig:imperfect-label}

\end{figure*}

\section{Conclusion and Future Work}
\label{section:conclusion}
We propose TrajectoryNet---a neural network architecture for point-based trajectory classification to infer real world human transportation modes from GPS traces.
To overcome the challenge of capturing the semantics of low-dimensional and heterogeneous feature space imposed by GPS data, we develop a novel representation that embeds the original feature space into another space that can be understood as a form of basis expansion.
The embedding can be viewed as a form of basis expansion that improves the predictive power in the original feature space.
The embedding can also be viewed as a mixture of experts that specialize in different areas of the decision boundary with different nonlinear functions.
We also employ segment-based features to enrich the feature space and use Maxout activations to improve the expressive power of RNNs' memory states.
Our experiments demonstrate that the proposed model achieves substantial improvements over the baseline results with over 98\% and 97\% classification accuracy when detecting 4 and 7 types of transportation modes.

As for future work, we consider incorporating location-based prior knowledge such as GIS information, developing online classification systems and building user-dependent profiles to further improve this system.
We also consider applying the proposed embedding method on other types of low-dimensional and heterogeneous time series data, e.g. Internet of things, to further explore the effectiveness of the proposed embedding method.

\begin{acks}
The authors acknowledge the support of the NSERC for this research.
\end{acks}

\bibliographystyle{ACM-Reference-Format}
\bibliography{bib.bib} 

\end{document}